\documentclass[journal]{IEEEtran}
\IEEEoverridecommandlockouts	 
\usepackage{cite}
\usepackage{bm}
\usepackage{amsthm}
\usepackage{amsmath,amssymb,amsfonts}
\usepackage[normalem]{ulem}
\usepackage{textcomp}
\usepackage{graphicx,color}
\usepackage{mathrsfs}
\usepackage[vlined,ruled,linesnumbered]{algorithm2e}
\usepackage{comment}
\usepackage{float}
\usepackage{soul}
\usepackage{url}
\usepackage{color}
\usepackage{dsfont}
\usepackage{bbm}
\usepackage{booktabs}
\usepackage{array}
\usepackage[table]{xcolor} 
\usepackage{yfonts}
\usepackage[normalem]{ulem}

\usepackage{tikz}
\usepackage{pgfplots}
\usepgfplotslibrary{fillbetween}

\newcommand{\stkout}[1]{\ifmmode\text{\sout{\ensuremath{#1}}}\else\sout{#1}\fi}
\usepackage{arydshln}
\usepackage{multicol}
\usepackage{amsmath}

\newtheorem{theorem}{Theorem}[section]
\newtheorem{lemma}[theorem]{Lemma}
\newtheorem{definition}{Definition}

\newtheorem{proposition}[theorem]{Proposition}

\newcommand{\until}[1]{\{1,\dots,#1\}}

\usepackage{tabularx}
\usepackage{multirow}
\usepackage{makecell}

\newcommand{\tguomargin}[1]{\marginpar{\color{purple}\tiny\ttfamily{TGUO:} #1}}

\newcommand\aamsout{\bgroup\markoverwith{\textcolor{violet}{\rule[0.5ex]{2pt}{1pt}}}\ULon}

\newcommand{\T}{\mathsf{T}} 

\newcommand{\mc}{\mathcal}

\DeclareSymbolFont{bbold}{U}{bbold}{m}{n}
\DeclareSymbolFontAlphabet{\mathbbold}{bbold}

\newcommand{\vect}[1]{\mathbbold{#1}}

\newcommand{\R}{\mathbb{R}}

\newcommand\oprocendsymbol{\hbox{$\square$}}
\newcommand\oprocend{\relax\ifmmode\else\unskip\hfill\fi\oprocendsymbol}

\renewcommand{\theenumi}{(\roman{enumi})}

\newcommand*{\QEDA}{\hfill\ensuremath{\blacksquare}}%

\makeatletter
\let\NAT@parse\undefined
\makeatother

\usepackage[urlcolor=blue,linkcolor=blue, citecolor=blue]{hyperref}

\graphicspath{{Images/}}

\title{\LARGE \bf Oscillator-Based Associative Memory with Exponential Capacity:\\ Theory, Algorithms, and Hardware Implementation}

\author{Arie Ogranovich$^{\dagger}$,  Taosha Guo$^{\dagger}$, Arvind
  R. Venkatakrishnan, Madelyn Shapiro,\\  Francesco Bullo, and Fabio Pasqualetti 
\thanks{This material is based upon work supported in part by ARO award W911NF-24-1-0228.
T. Guo and F. Pasqualetti are with the
Department of Electrical Engineering and Computer Science at the 
University of California, Irvine, Irvine, CA, 92697, USA. Emails: \href{mailto:taoshag@uci.edu}{\{\texttt{taoshag}},
\href{mailto:fabiopas@uci.edu}{\texttt{fabiopas\}@uci.edu}},
A. Ogranovich, A. R. Venkatakrishnan, M. Shapiro, and F. Bullo are with the
Department of Mechanical Engineering at the 
University of California, Santa Barbara, Santa Barbara, 93106, USA.
\href{mailto:arieogranovich@ucsb.edu}
{\{\texttt{arieogranovich}},\href{mailto:arvindragghav@ucsb.edu}
{\{\texttt{arvindragghav}},\href{mailto:madelynshapiro@ucsb.edu}
{\{\texttt{madelynshapiro}},
\href{mailto:bullo@ucsb.edu}{\texttt{bullo\}@ucsb.edu}}. We acknowledge assistance from Gemini in providing LaTeX code to format Fig. 6.
}
}

\usepackage{tikz}
\usetikzlibrary{positioning}

\begin{document}

\maketitle
\thispagestyle{empty}

\begin{abstract}
  Associative memory systems enable content-addressable storage and
  retrieval of patterns, a capability central to biological
  neural computation and artificial intelligence. Classical
  implementations such as Hopfield networks face fundamental
  limitations in memory capacity, scaling at most linearly with
  network size. We present an associative memory architecture based on
  Kuramoto oscillator networks with honeycomb topology in which memories
  are encoded as stable phase-locked configurations. The honeycomb
  network consists of multiple cycles that share nodes in a chain-like
  arrangement, creating a one-dimensional lattice of chained
  +loops. We prove that this architecture achieves exponential memory
  capacity: a network of $N$ oscillators can store
  $(2\lceil n_c/4 \rceil - 1)^m$ distinct patterns, where $m$
  honeycomb cycles each contain $n_c$ oscillators. Moreover, we
  fully characterize all stable configurations and prove that
  each memory's basin of attraction maintains a guaranteed minimum
  size independent of network scale. Simulations using
  charge-density-wave (CDW) oscillators validate predicted
  phase-locking behavior, demonstrating practical realizability in
  neuromorphic hardware.
\end{abstract}

\begingroup
\renewcommand\thefootnote{}
\footnotetext{$^\dagger$ These authors contributed equally.}
\endgroup

\section{Introduction}\label{sec:intro}
Associative memory, the ability to store and retrieve patterns through
partial or noisy cues, is a fundamental computational primitive in
both biological and artificial neural systems
\cite{JJHopfield:82}. Unlike conventional address-based memory where
storage locations are accessed through explicit indices, associative
memory enables content-addressable recall: presenting a corrupted or
incomplete pattern causes the system to converge to the stored memory
that most closely matches the input. This capability is essential for
pattern recognition \cite{krotov2016DenseAssociativeMemory}, error
correction, and cognitive functions such as categorization and
generalization from examples. The development of efficient, scalable
associative memory architectures remains a central challenge in
neuromorphic computing and brain-inspired artificial intelligence
\cite{kudithipudi2025NeuromorphicComputingScale}.

The Hopfield network \cite{JJHopfield:82}, a canonical model of
associative memory, stores patterns as stable fixed points of a
recurrent neural network with symmetric connections. While elegant in
its simplicity, the standard Hopfield model suffers from severe
capacity limitations: a network of $N$ neurons can reliably store only
$O(N/\log N)$ patterns before spurious attractors and retrieval errors 
dominate \cite{LP-IG-DG:86}. Recent advances have pushed this bound
higher through modified energy functions and higher-order interactions
\cite{demircigil2017exponentialAssociativeMemory,
  krotov2016DenseAssociativeMemory,
  hillar2018RobustExponentialMemory}, achieving polynomial or even
exponential capacity in specialized settings. However, these
approaches often require dense all-to-all connectivity or complex
training procedures that are challenging to implement in
energy-efficient hardware. 

Oscillator-based associative memories offer an alternative paradigm
that is naturally suited to physical implementation. Networks of
coupled oscillators—realized through spin-torque nano-oscillators
\cite{torrejon2017NeuromorphicComputingNanoscale}, optical resonators
\cite{musa2025DenseAssociativeMemory}, or electronic circuits
\cite{choi2025HardwareImplementationRing}—can encode memories as
stable phase configurations, where the relative phases of oscillators
represent stored patterns. Early work demonstrated that oscillatory
networks can achieve perfect retrieval with specific connectivity
patterns \cite{TN-FCH-YCL:04, TN-YCL-FC:04, izhikevich1999weakly}, but
questions of capacity scaling and basin of attraction robustness
remained largely open. The Kuramoto model \cite{Doerfler2014}, which
describes the dynamics of weakly coupled phase
oscillators, provides a tractable mathematical framework for analyzing
such systems.

Several recent works have explored oscillator-based associative memory
architectures. Nishikawa et al. \cite{TN-FCH-YCL:04, TN-YCL-FC:04}
demonstrated perfect retrieval in specific oscillator network
topologies but did not provide capacity scaling results. Hölzel and
Krischer \cite{RWH-KK:15} and Zhao et al. \cite{XZ-ZL-XX:20} analyzed
Hebbian learning rules for Kuramoto networks, focusing on stability of
learned patterns rather than fundamental capacity limits. More
recently, Nagerl and Berloff
\cite{nagerl2025HigherOrderKuramotoOscillator} proposed higher-order
Kuramoto models for dense associative memory, achieving capacity
scaling through increased coupling complexity. However, these
approaches either require carefully designed learning algorithms,
higher-order interactions, or do not provide guarantees on the memory
capacity or basin of attraction sizes. In contrast, our work
establishes exponential capacity with fixed, local coupling in a
simple honeycomb topology, and proves that basin volumes maintain a
guaranteed minimum size independent of network scale—these are
essential properties for robust hardware implementation.

In this paper, we demonstrate that Kuramoto oscillator networks with
honeycomb topology achieve exponential memory capacity while
maintaining desirable stability properties. The honeycomb network,
consisting of $m$ cycles of size $n_c$ that share nodes in a
chain-like arrangement, provides a sparse connectivity structure that
is amenable to hardware implementation while enabling rich
phase-locking dynamics. Unlike densely connected networks, the
honeycomb topology requires only local coupling between oscillators,
reducing wiring complexity and energy consumption when implemented in
hardware. We show that this simple structure supports
$(2\lceil n_c/4 \rceil - 1)^m$ distinct stable phase-locked
configurations, exponentially scaling with the number of cycles
$m$. Moreover, each stable configuration possesses a basin of
attraction with guaranteed minimum size that does not degrade as the
network grows—a critical property for reliable pattern retrieval in
the presence of noise or perturbations.

The main contributions of this paper are as follows. First, we prove that
Kuramoto oscillator networks on honeycomb graphs achieve exponential memory
capacity, storing $(2\lceil n_c/4 \rceil - 1)^m$ distinct patterns in a network
with $m$ cycles of size $n_c$. Second, we provide a complete characterization of
all stable phase-locked configurations, showing that stability is determined by
phase-cohesiveness: adjacent oscillators must have phase differences
less than $\pi/2$.  Third, we provide explicit encoding and decoding schemes to
map discrete memory patterns to unique phase-locked configurations. Fourth, we
show that each memory's basin of attraction maintains a guaranteed minimum
volume independent of network size, ensuring robust retrieval even as capacity
scales.  Finally, we validate the theoretical predictions through simulations of
charge-density-wave oscillator networks, demonstrating practical realizability
in neuromorphic hardware. A preliminary version of these results was presented
in \cite{guo2025OscillatoryAssociativeMemory}. Compared to
\cite{guo2025OscillatoryAssociativeMemory}, here we provide a full, revised
proof of the stable-configuration characterization, a novel analysis of the
basin of attraction of the desired phase-locked configurations, an updated and
more efficient algorithm to encode and decode memories as stable phase-locked
configurations of the network, and describe a pathway for hardware
implementation of the proposed associative memory scheme.

The remainder of this paper is organized as
follows. Section~\ref{sec:preliminaries} introduces the Kuramoto
oscillator model and the honeycomb network
topology. Section~\ref{sec:main_results} presents our main capacity
theorem, encoding scheme, and basin of attraction
analysis. Section~\ref{sec:cdw} describes the CDW oscillator
implementation and numerical validation. Section~\ref{sec:conclusion}
concludes the paper and discusses directions for future work.

\section{Preliminaries}\label{sec:preliminaries}

\subsection{Kuramoto Oscillator Networks}
Consider a network of $n$ coupled phase oscillators described by the
Kuramoto model:
\begin{equation}\label{eq:kuramoto}
  \dot{\theta}_i = \omega_i + \sum_{j=1}^n a_{ij} \sin(\theta_j -
  \theta_i), \quad i = 1, \ldots, n ,
\end{equation}
where $\theta_i \in \mathbb{S}^1$ denotes the phase of oscillator $i$,
$\omega_i \in \mathbb{R}$ is its natural frequency, and
$a_{ij} \geq 0$ represents the coupling strength from oscillator $j$
to oscillator $i$. The adjacency matrix $A = [a_{ij}]$ encodes the
network topology: $a_{ij} > 0$ if oscillators $i$ and $j$ are
connected, and $a_{ij} = 0$ otherwise.

Throughout this paper, unless specified differently, we focus on the
case of identical natural frequencies, setting $\omega_i = 0$ for all
$i$ without loss of generality (any common frequency can be removed by
a rotating frame transformation), and symmetric adjacency matrix
(undirected coupling, i.e., $a_{ij} = a_{ji} = 1$ for all
$i,j$).\footnote{The choice of unit weights is for presentation purposes; our results apply when the weights are equal and positive.} The dynamics then simplify~to:
\begin{equation}\label{eq:kuramoto_identical}
  \dot{\theta}_i = \sum_{j=1}^n a_{ij} \sin(\theta_j - \theta_i), \quad
  i = 1, \ldots, n .
\end{equation}

A \emph{phase-locked configuration}
$\theta^* = (\theta_1^*, \ldots, \theta_n^*)$ is an equilibrium of
\eqref{eq:kuramoto_identical} satisfying $\dot{\theta}_i = 0$ for all
$i$. At such configurations, oscillators maintain fixed phase
relationships. We say that a phase-locked configuration is \emph{phase-cohesive} if all phase differences are less than $\frac{\pi}{2}$, i.e., if $\theta_i-\theta_j \in ]-\frac{\pi}{2},\frac{\pi}{2}[$ for all oscillators $i,j$.

\subsection{Honeycomb Network Topology}
We now define the honeycomb network structure that forms the basis of
our associative memory architecture.



\begin{definition}{\emph{\textbf{\emph{(Honeycomb
          graph)}}}}\label{def:honeycomb}
  For integers $m \geq 1$ and $n_c \geq 5$, the \emph{honeycomb graph}
  $\mathcal{G}_m^{n_c}$ is constructed as follows:
  \begin{enumerate}
  \item Create $m$ cycles, denoted
    $\mathcal{C}_1, \mathcal{C}_2, \ldots, \mathcal{C}_m$, where
    each cycle initially contains $n_c$ nodes.
  \item For each $p \in \{1, \ldots, m-1\}$, merge one node from
    cycle $\mathcal{C}_p$ with one node from cycle $\mathcal{C}_{p+1}$ 
    into a single shared node.
  \end{enumerate}
  The resulting network has $n = m(n_c - 1) + 1$ nodes and forms a
  one-dimensional chain of interlocking cycles. \oprocend
\end{definition}

\begin{figure}[t]
  \centering
  \includegraphics[width=0.85\columnwidth,trim={0cm 0cm 0cm
    0cm},clip]{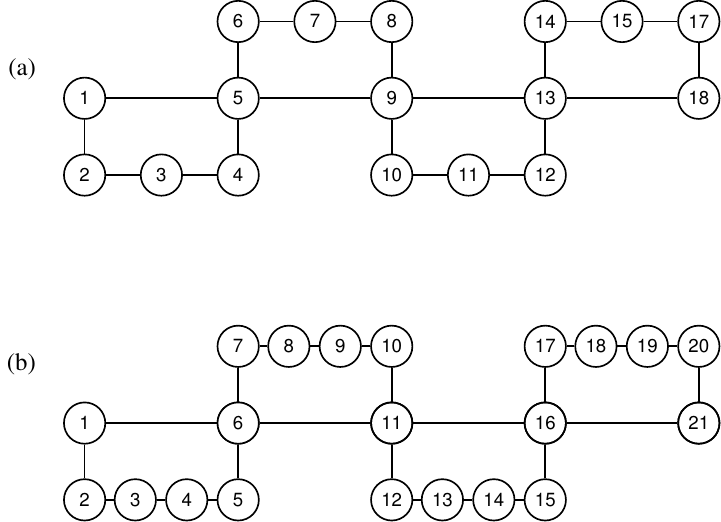}
  \caption{This figure shows the 1D honeycomb network described by the
    graph $\mc H_m^{n_\text{c}}$. In (a), the network has $m=5$
    cycles, and each cycle is a pentagon with
    $n_c =5$ nodes. In (b), the network has $m=5$ cycles, and each cycle is a pentagon $n_c = 6$ nodes.}
  \label{fig:honeycomb_example}
\end{figure}

Figure~\ref{fig:honeycomb_example} illustrates a honeycomb network
with $m=3$ cycles of size $n_c=5$, resulting in $n = 13$ total nodes.
While various node numbering schemes and choices of which nodes to
share between cycles are possible when constructing honeycomb
networks, throughout this paper our examples and numerical experiments will be of honeycomb networks with the following particular structure:
cycle $\mathcal{C}_p$ contains nodes $(p-1)(n_c-1) + 1$ through
$p(n_c-1) + 1$, with the last node of cycle $\mathcal{C}_p$ being the
first node of cycle $\mathcal{C}_{p+1}$.  This numbering and
connection pattern is illustrated in
Figure~\ref{fig:honeycomb_example}. The honeycomb topology provides
sparse local connectivity while enabling rich synchronization dynamics
in coupled oscillator networks.  As we show in
Section~\ref{sec:main_results}, each cycle can encode independent
phase patterns, and the shared nodes couple these patterns together,
allowing the oscillator network to store and retrieve complex
multi-cycle configurations. 

\subsection{Associative Memory via Phase-Locked Configurations}
In classical Hopfield networks \cite{JJHopfield:82}, memories are
directly encoded as binary vectors $\{-1, +1\}^N$, with each neuron's
state corresponding to one bit. The network is typically trained using
Hebbian learning rules to make desired binary patterns stable fixed
points of the dynamics. During recall, the network is initialized with
a partial or corrupted version of a stored pattern, and the dynamics
drive the system toward the nearest stored binary pattern, which
serves as a stable attractor. This content-addressable recall
mechanism enables robust pattern retrieval despite input noise or
incompleteness.


In oscillator-based associative memory, we instead encode memories as
stable phase-locked configurations of the Kuramoto network. However,
unlike Hopfield networks where the correspondence between network
states and memories is immediate, phase-locked configurations are
characterized by continuous phase variables and their
differences. This necessitates a mapping between the space of
phase-locked configurations—or equivalently, the pattern of phase
differences across the network—and the discrete space of memories to
be stored. The design of such an encoding scheme requires addressing
several fundamental questions: characterizing how many distinct stable
phase-locked configurations a given network topology can support,
establishing an explicit mapping that associates each memory with a
unique stable configuration, and analyzing the basin of attraction
properties that ensure reliable retrieval from partial or noisy
inputs. The remainder of this paper addresses these questions for
Kuramoto networks on honeycomb graphs. 

\section{Main Results}\label{sec:main_results}

\subsection{Exponential Memory Capacity}
We now present our main result: Kuramoto oscillator networks on
honeycomb graphs achieve exponential memory capacity. We establish the
exact number of distinct stable phase-locked configurations and
provide an explicit encoding scheme that maps discrete memories to
these configurations.

\begin{theorem}{\emph{\textbf{\emph{(Exponential capacity of honeycomb Kuramoto
          networks)}}}}\label{thm:capacity}
  Consider the Kuramoto network \eqref{eq:kuramoto_identical} on the
  honeycomb graph $\mathcal{G}_m^{n_c}$ with cycles
  $\mc C_1,\ldots, \mc C_m$. Then,
  \begin{enumerate}
    
  \item the network has a set $\mc S$ of $N_{\text{eq}}$ locally
    exponentially stable phase-locked configurations (up to rotational
    equivalence), where
    \begin{equation}\label{eq:capacity}
      N_{\text{eq}} = \left(2\left\lceil \frac{n_c}{4} \right\rceil - 1\right)^m;
    \end{equation}
    
  \item each phase-locked configuration in $\mathcal{S}$ features
    constant phase differences within each cycle, with adjacent oscillators in
    $\mathcal{C}_p$, $p = 1, \ldots, m$, having phase difference
    $\frac{2\pi}{n_c} k_p$ for some integer
    $k_p \in \{-\lceil n_c/4 \rceil + 1, \ldots, \lceil n_c/4 \rceil -
    1\}$;

      \item all phase-locked configurations not in $\mc S$ are unstable.
  \end{enumerate}
\end{theorem}

\bigskip A proof of Theorem \ref{thm:capacity} is postponed to the
Appendix. To illustrate the exponential scaling in
Theorem~\ref{thm:capacity}, consider a honeycomb network with
$n_c = 5$. Since $\lceil 5/4 \rceil = 2$, there are
$(2 \cdot 2 - 1)^m = 3^m$ stable phase-locked configurations. The
network has $n = m(n_c - 1) + 1 = 4m + 1$ nodes, so the number of
stable configurations is $3^m = 3^{(n-1)/4} \approx 1.32^{n-1}$,
growing exponentially with network size. Figure~\ref{fig:phase_lock_0}
illustrates the three stable configurations for a honeycomb network
with $n_c = 5$ and a single cycle ($m=1$).
\begin{figure}[t]
  \centering
  \includegraphics[width=1\columnwidth]{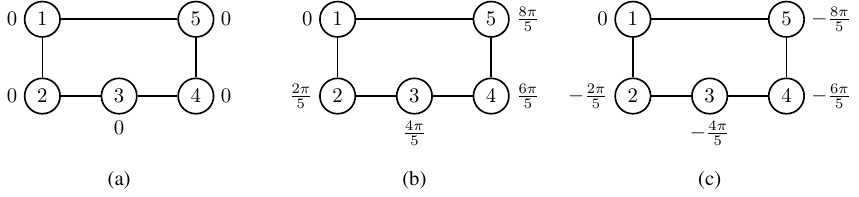}
  \caption{Stable phase-locked configurations for a single cycle with 
    $n_c=5$ ($m=1$). (a) Synchronization: all phases equal. (b) Phase 
    differences of $-2\pi/5$ between consecutive oscillators. (c) Phase 
    differences of $+2\pi/5$ between consecutive oscillators.}
  \label{fig:phase_lock_0}
\end{figure}

More generally, the memory capacity—defined as the ratio of storable
patterns $N_\text{eq}$ to network size $n$—is given by
\begin{equation}\label{eq:capacity_ratio}
  C = \frac{N_\text{eq}}{n} = \frac{\bigl(2\lceil n_c/4 \rceil - 1
    \bigr)^{(n-1)/(n_c-1)}}{n} . 
\end{equation}
This exponential scaling stands in sharp contrast to the classical
Hopfield network, where capacity is limited to $O(n/\log n)$ patterns
\cite{LP-IG-DG:86}. For the honeycomb networks with $n_c = 5$ and
$n = 4m +1$, capacity grows as $C \approx 1.32^{n-1}/n$, meaning that
even modest networks can store vastly more patterns than their
Hopfield counterparts of comparable size.

\subsection{Encoding and Decoding Scheme}
To use the oscillators' stable phase-locked configurations as
associative memory, we require a bijective mapping between discrete
memory patterns and stable configurations. In this work, we represent
memories as integers $s \in \{0, 1, \ldots, N_\text{eq} - 1\}$. For
practical applications, arbitrary binary data can be stored by first
converting it to an integer in this range using standard
binary-to-decimal conversion. We now present the encoding and decoding
scheme.

Each stable configuration in $\mathcal{S}$ is uniquely characterized
by an $m$-tuple $(k_1, \ldots, k_m)$ where
$k_p \in \{-\lceil n_c/4 \rceil + 1, \ldots, \lceil n_c/4 \rceil -
1\}$ specifies the constant phase difference pattern along cycle
$\mathcal{C}_p$ (see Theorem~\ref{thm:capacity}(ii)). Since each
component $k_p$ can take $b = 2\lceil n_c/4 \rceil - 1$ distinct
values, there are exactly $b^m = N_\text{eq}$ possible tuples,
establishing a natural correspondence between integers
$s \in \{0, 1, \ldots, N_\text{eq} - 1\}$ and stable configurations
via base-$b$ representation.

\smallskip
\noindent
\textbf{Encoding:} Given a memory index
$s \in \{0, 1, \ldots, N_\text{eq} - 1\}$, we construct the
corresponding phase-locked configuration $\theta^*$ as follows:
\begin{enumerate}
\item convert $s$ to base $b = 2\lceil n_c/4 \rceil - 1$, obtaining
  digits $\bar{s}_1, \ldots, \bar{s}_m$ where each
  $\bar{s}_p \in \{0, 1, \ldots, b-1\}$;
\item set $k_p = \bar{s}_p - (\lceil n_c/4 \rceil - 1)$ for each cycle
  $p$, shifting to the range
  $\{-\lceil n_c/4 \rceil + 1, \ldots, \lceil n_c/4 \rceil - 1\}$;
\item construct the phase configuration by setting $\theta^*_1 = 0$
  and defining phases sequentially. For each cycle $\mathcal{C}_p$
  with nodes indexed from $(p-1)(n_c-1) + 1$ to $p(n_c-1) + 1$, set
  consecutive node phases to satisfy
  $\theta^*_{i+1} - \theta^*_i = \frac{2\pi}{n_c} k_p$ for
  $i = (p-1)(n_c-1) + 1, \ldots, p(n_c-1)$.
\end{enumerate}

\smallskip
\noindent
\textbf{Decoding:} Given a phase-locked configuration $\theta^*$, we
extract the memory index $s$ as follows:
\begin{enumerate}
\item for each cycle $\mathcal{C}_p$ with $p \in \{1, \ldots, m\}$,
  compute $k_p = \frac{n_c}{2\pi}(\theta^*_{i+1} - \theta^*_i)$ where
  $i = (p-1)(n_c-1) + 1$ (the first node of cycle $\mathcal{C}_p$);
\item set $\bar{s}_p = k_p + (\lceil n_c/4 \rceil - 1)$ for each $p$, 
  shifting from $\{-\lceil n_c/4 \rceil + 1, \ldots, \lceil n_c/4 \rceil - 1\}$ 
  to $\{0, 1, \ldots, b-1\}$;
\item Convert the base-$b$ representation $(\bar{s}_1, \ldots, \bar{s}_m)$ 
  to decimal to obtain the memory index $s$.
\end{enumerate}

Figure~\ref{fig:phase_lock} illustrates two different encoded memories for a honeycomb network with $m=2$ cycles of size $n_c=5$. 
\begin{figure}[t]
  \centering
  \includegraphics[width=0.85\columnwidth]{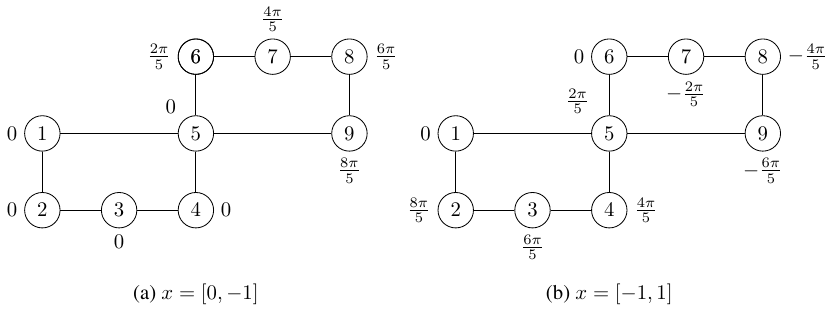}
  \caption{Stable phase-locked configurations encoding two different
    memory patterns for a Kuramoto honeycomb network with $m=2$ cycles
    and $n_c=5$. (a) Configuration with $(k_1, k_2) = (0, -1)$: the
    first cycle has equal phases while the second has phase
    differences of $-2\pi/5$. (b) Configuration with
    $(k_1, k_2) = (1, 1)$: both cycles have phase differences of
    $+2\pi/5$.}
  \label{fig:phase_lock}
\end{figure}

\subsection{Basin of Attraction Analysis}
A critical requirement for practical associative memory is robustness:
the network must reliably retrieve stored patterns even when
initialized with noisy or partial inputs. In our framework, successful
retrieval means that a perturbed initial condition converges back to
the original stored configuration, a notion that is quantified by the
basin of attraction of that configuration. We now establish that each
stable phase-locked configuration has a basin of attraction with
guaranteed minimum size, ensuring that retrieval remains successful
for perturbations below a specific threshold at each oscillator.

Two phase configurations $\theta$ and $\theta'$ are called
\emph{rotationally equivalent} if $\theta'_i = \theta_i + c$ for all
$i$ and some constant $c \in \mathbb{R}$. Since the Kuramoto dynamics
\eqref{eq:kuramoto_identical} and our encoding/decoding schemes depend
only on phase differences, rotationally equivalent configurations
represent the same physical state.

\begin{definition}{\emph{\textbf{\emph{(Basin of
          attraction)}}}}\label{def:basin}
  The basin of attraction of a phase-locked configuration $\theta^*$
  is the set of all initial conditions $\theta(0)$ such that the
  trajectory of the Kuramoto network \eqref{eq:kuramoto_identical}
  converges to a phase-locked configuration rotationally equivalent to
  $\theta^*$. \oprocend
\end{definition}

The next result establishes a guaranteed size of the basin of
attraction that does not degrade as the network grows.

\begin{theorem}{\emph{\textbf{\emph{(Basin of attraction for honeycomb Kuramoto
          networks)}}}}\label{thm:basin}
  Let $\theta^*$ be any stable phase-locked configuration of the
  Kuramoto network on $\mathcal{G}_m^{n_c}$. All initial conditions
  $\theta(0)$ satisfying
  \begin{equation}\label{eq:basin_bound}
    \max_i |\theta_i(0) - \theta^*_i| < \frac{\pi}{2} - \left( \left\lceil
      \frac{n_c}{4} \right\rceil  - 1 \right) \frac{2\pi}{n_c}  = \frac{\pi
      r}{2n_c} ,
  \end{equation}
  where $r = ((n_c - 1) \bmod 4) + 1$, converge to a phase-locked
  configuration rotationally equivalent to $\theta^*$.  \oprocend
\end{theorem}

A proof of Theorem \ref{thm:basin} is provided in the Appendix. This
result has two important implications. First, the radius of the basin
is independent of the number of cycles $m$: adding more cycles to
increase memory capacity does not reduce the robustness of individual
memories. Second, the basin size scales as $O(1/n_c)$, with the
guaranteed radius lying in the range
$\frac{\pi}{2n_c} \leq \frac{\pi r}{2n_c} \leq \frac{2\pi}{n_c}$
depending on $n_c$. The largest basins occur for cycle sizes that are
multiples of 4 (e.g., $n_c = 8, 12, 16$), while the smallest occur for
$n_c = 5, 9, 13, ...$. For concrete examples, a network with $n_c = 5$
guarantees successful retrieval when each oscillator is perturbed by
up to $\frac{\pi}{10} \approx 0.31$ radians, while $n_c = 8$ provides
a larger guarantee of $\frac{\pi}{4} \approx 0.79$ radians. These are
substantial perturbation tolerances: for $n_c = 8$, each oscillator
can be perturbed by nearly a quarter of its full $2\pi$ state space
while still ensuring convergence to the correct memory.


\begin{figure}[t]
  \centering
  \begin{tikzpicture}
    \node[anchor=south west,inner sep=0] (image) at (0,0)
    {\includegraphics[width=0.95\linewidth]{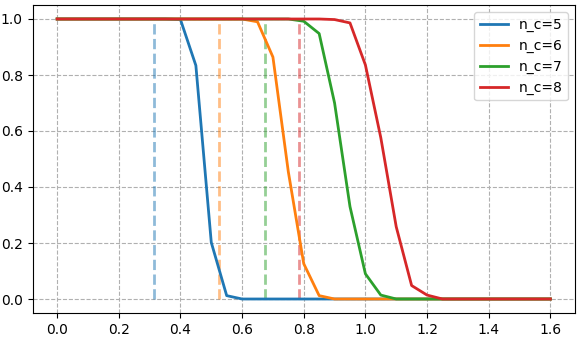}};
    
    \begin{scope}[x={(image.south east)},y={(image.north west)}]
      \node[below] at (0.5,0) {\scriptsize Perturbation magnitude (radians)};


      \node[rotate=90, above] at (0,0.5) {\scriptsize Retrieval success rate};
      
      
    \end{scope}
  \end{tikzpicture}
  \caption{Numerical validation of basin of attraction bounds for
    honeycomb networks with $m=400$ cycles. The solid lines show the
    empirical success rate of memory retrieval under random
    perturbations of varying magnitude. The vertical dashed lines mark
    the theoretical lower bounds from Theorem~\ref{thm:basin}. For
    each cycle size, we randomly selected 500 stable phase-locked
    configurations and applied independent random perturbations at
    each magnitude level, then simulated the Kuramoto dynamics to
    determine if the system converged to the original phase-locked
    configuration.}
  \label{fig:basin_numerics}
\end{figure}

To validate these theoretical bounds and assess their tightness, we
performed numerical simulations on honeycomb networks with $m = 400$
cycles for various cycle sizes. Figure~\ref{fig:basin_numerics} shows
the empirical success rate of memory retrieval as a function of
perturbation magnitude. The numerical experiments reveal that the
theoretical bounds are reasonably tight: the empirical basin radii are
approximately 0.40, 0.65, 0.80, and 0.95 radians for
$n_c = 5, 6, 7, 8$ respectively, compared to theoretical guarantees of
0.31, 0.52, 0.67, and 0.79 radians. The theoretical bounds are thus
within 23\% of the numerically observed basin sizes, indicating that
Theorem~\ref{thm:basin} provides practically useful guarantees rather
than overly conservative estimates.

These results reveal a fundamental tradeoff between memory capacity
and robustness. Comparing $n_c = 5$ and $n_c = 8$: the Kuramoto
honeycomb network with $n_c = 8$ provides basin radii approximately
twice as large as the case with $n_c = 5$, offering superior noise
tolerance. However, this comes at the cost of reduced memory
capacity. The network with $n_c = 8$ ($n = 7m + 1$) has a memory
capacity
$$C_8 = \frac{3^{(n-1)/7}}{n} \approx \frac{1.17^{n-1}}{n},$$
while the network with $n_c = 5$ ($n = 4m + 1$) features
$$C_5 = \frac{3^{(n-1)/4}}{n} \approx \frac{1.32^{n-1}}{n} .$$
The choice of $n_c$ thus allows designers to balance storage density
and retrieval robustness based on application needs.

\section{Hardware Implementation with Charge-Density-Wave
  Oscillators}\label{sec:cdw}
A key advantage of oscillator-based associative memory is the potential for
physical implementation using limit cycle oscillators. Many physical oscillator
systems—including van der Pol oscillators and charge-density-wave (CDW)
oscillators—exhibit synchronization behaviors analogous to Kuramoto oscillators
\cite{brown2004phase,nakao2016phase,brown2025charge}. Under weak coupling
conditions, the phase dynamics of these limit cycle oscillators can be reduced
to Kuramoto-type equations, allowing the theoretical guarantees established in
Section~\ref{sec:main_results} to be realized in hardware. We now demonstrate
this realizability through simulations of coupled CDW oscillator networks. Our
analysis uses the results developed in \cite{brown2025charge}, to which we refer
the reader for all derivations and a description of the underlying physics.

CDW devices based on 1T-TaS$_2$ are nanoscale quantum materials that exhibit
hysteretic switching between high-resistance and low-resistance states due to
electric-field-induced phase transitions. When a CDW device is biased with DC
voltage $V_{\text{DC}}$ through a series load resistor $R_S$,
the current-voltage hysteresis produces stable limit-cycle oscillations in the
output voltage $V_O(t)$.
The output voltage dynamics are
described by
\begin{equation}\label{eq:v_dynamics}
  \dot{V}_O(t) =
  \begin{cases}
    -\frac{1}{R_{\text{ch}} C} V_O(t) + \frac{1}{R_{\text{ch}} C} V_{\text{ch}}
    & \text{(charging)},\\
    -\frac{1}{R_{\text{dis}} C} V_O(t) + \frac{1}{R_{\text{dis}} C}
    V_{\text{dis}} & \text{(discharging)},
  \end{cases}
\end{equation}
where $C$ is the parasitic capacitance, $R_H$ and $R_L$ are the high and low
resistance states, and the effective parameters~are
\begin{align*}
  R_{\text{ch}} &= \frac{R_H R_S}{R_H + R_S}, \quad V_{\text{ch}} =
                  \frac{R_H}{R_H + R_S} V_{\text{DC}},\\ 
R_{\text{dis}} &= \frac{R_L R_S}{R_L + R_S}, \quad V_{\text{dis}} =
                 \frac{R_L}{R_L + R_S} V_{\text{DC}}. 
\end{align*}


To implement the honeycomb network topology, we couple CDW oscillators through
resistors $R_{ij}$. 
The voltage dynamics of the $i$-th oscillator in the coupled network~are
\begin{equation}\label{eq:coupled_voltage}
  \dot{V}_{O_i}(t) = f(V_{O_i}(t)) - \sum_{j=1}^{n} \frac{1}{R_{ij}C} (V_{O_i}(t) - V_{O_j}(t)),
\end{equation}
where $f(V_O)$ represents the free-running dynamics in
\eqref{eq:v_dynamics}, and the coupling term arises from currents
$I_{ij} = (V_{O_i} - V_{O_j})/R_{ij}$ flowing through the coupling
resistors.



By appropriately selecting the load resistance  of the CDW oscillator \cite{brown2025charge}, its voltage waveform $V_O(t)$ exhibits periodic oscillations. The instantaneous phase $\theta_i(t)$ of the $i$-th oscillator can be defined such that $\theta_i(t)$ increases uniformly in time in the absence of coupling, i.e., $\dot{\theta}_i = \omega_0 = 1/(T_{\text{ch}} + T_{\text{dis}})$, where $T_{\text{ch}}$ and $T_{\text{dis}}$ are the charging and discharging times, respectively, within one limit cycle.

{When coupled to the network via resistors $R_{ij}$, the phase evolution is governed by
\begin{align*}
\dot{\theta}_i(t) = \omega_0 + \sum_{j=1}^{N} \frac{1}{R_{ij}C} \Gamma(\theta_j(t) - \theta_i(t)),
\end{align*}
where $\Gamma(\phi)$ is the phase interaction function defined by averaging the perturbation effect over one cycle:
\begin{equation}\label{eq: Gamma function}
\Gamma(\phi) = \frac{1}{2\pi}\int_{0}^{2\pi} Z(\phi + \psi) \left( V_O(\phi + \psi) - V_O(\psi) \right) \mathrm{d}\psi.
\end{equation}
Here, $Z(\cdot)$ denotes the $2\pi$-periodic phase response curve, which quantifies the oscillator's sensitivity to small perturbations at different points along its limit cycle. As established in \cite{brown2025charge}, $Z(\cdot)$ can be determined numerically by applying impulsive perturbations to the limit cycle voltage and observing the induced phase shifts. Once $Z(\cdot)$ is obtained, the interaction function $\Gamma(\cdot)$ is computed via the integral in \eqref{eq: Gamma function}. Previous analysis \cite{brown2025charge} indicates that for suitably chosen circuit parameters, the interaction function is well-approximated by a sinusoidal form, i.e., $\Gamma(\phi) \approx A \sin(\phi)$. Consequently, the coupled CDW network dynamics effectively reduce to the Kuramoto model in \eqref{eq:kuramoto}, with coupling weights proportional to $(R_{ij}C)^{-1}$. This  allows us to leverage the theoretical guarantees derived for Kuramoto networks to interpret the synchronization behavior of physical CDW oscillator implementations.}

We now simulate a honeycomb network with $m=2$ cycles and $n_c=5$
using the CDW oscillator model. Circuit parameters are:
$R_H = 2.5\text{k}\Omega$, $R_L = 500\Omega$, $C = 10\text{pF}$,
$R_S = 1\text{k}\Omega$, $V_{\text{DC}} = 3\text{V}$, and
$R_{ij} = 3.5\text{k}\Omega$.

We initialize the network with voltages corresponding to perturbed
versions of two distinct target phase-locked configurations. For cycle
1 (nodes 1-5), the initial phases are set near
$[0, 2\pi/5, 4\pi/5, 6\pi/5, 8\pi/5]$ (the configuration with
$k_1 = 1$) by adding perturbations
$[0.064\pi, -0.049\pi, 0.065\pi, -0.002\pi, 0.082\pi]$. For cycle 2
(nodes 5-9, sharing node 5 with cycle 1), initial phases are set near
$[8\pi/5, 6\pi/5, 4\pi/5, 2\pi/5, 0]$ (the configuration with
$k_2 = -1$) with perturbations
$[0.082\pi, 0.085\pi, 0.044\pi, -0.088\pi]$. These perturbed initial
conditions test whether the network converges to the intended
phase-locked configurations.

Figure~\ref{fig:cdw_simu}(b,d) shows the steady-state voltage
waveforms. The time intervals $\Delta T_i$ between consecutive voltage
peaks are equal within each cycle ($\Delta T_1 = \cdots = \Delta T_5$
and $\Delta T_6 = \cdots = \Delta T_{10}$), confirming frequency
synchronization. Figure~\ref{fig:cdw_simu}(c,e) displays the phase
differences between consecutive oscillators after synchronization:
nodes 1-5 converge to phase differences of $+2\pi/5$, while nodes 6-10
converge to $-2\pi/5$, matching the theoretical predictions of
Theorem~\ref{thm:capacity}.

\begin{figure}[t]
    \centering
    \includegraphics[width=1\linewidth]{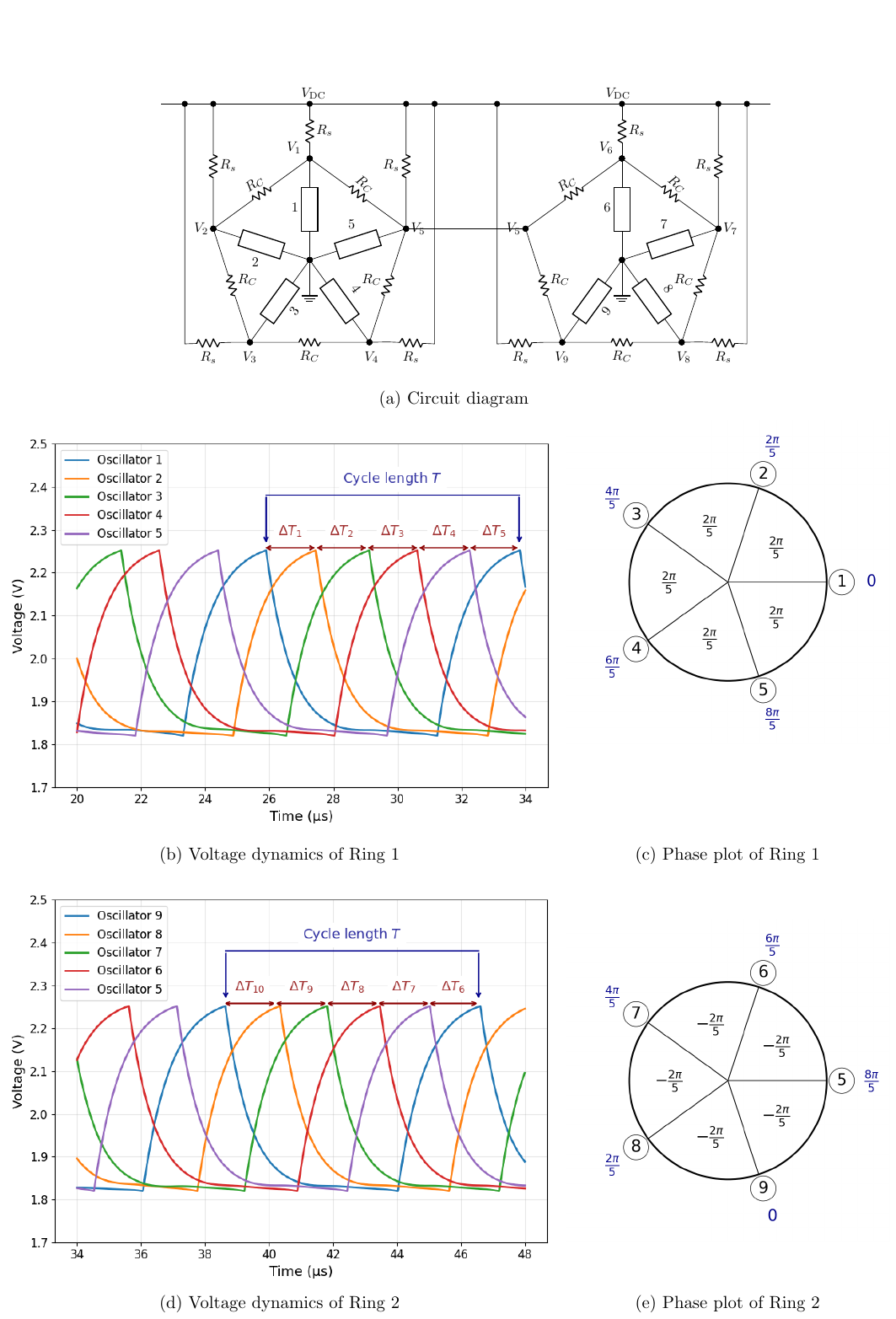}
    \caption{Numerical simulation of CDW oscillator network
      implementing a honeycomb topology with $m=2$ cycles and
      $n_c=5$. (a) Circuit topology showing two cycles sharing node
      5. (b,d) Steady-state voltage waveforms for cycles 1 and 2,
      demonstrating synchronized oscillations with characteristic
      phase delays. (c,e) Phase circle diagrams showing uniform
      $2\pi/5$ phase spacing: counterclockwise in cycle 1 and
      clockwise in cycle 2.}
    \label{fig:cdw_simu}
\end{figure}

These results demonstrate that the phase-locked configurations
predicted by our Kuramoto analysis can be realized in physical CDW
oscillator networks, validating the practical feasibility of
oscillator-based associative memory with exponential capacity and
robust retrieval guarantees.

\section{Conclusion}\label{sec:conclusion}
We have demonstrated that Kuramoto oscillator networks on honeycomb
graphs achieve exponential memory capacity for associative memory. A
honeycomb network with $m$ cycles of size $n_c$ supports
$(2\lceil n_c/4 \rceil - 1)^m$ distinct stable phase-locked
configurations, requiring only sparse local connectivity. We provided
a complete characterization of stable configurations through the
phase-cohesiveness property and an explicit encoding scheme mapping
discrete patterns to these configurations. Furthermore, we proved that
each memory's basin of attraction maintains a guaranteed minimum size
independent of network scale, ensuring robust retrieval when
oscillators are perturbed below a threshold of $\frac{\pi r}{2n_c}$
radians. Numerical simulations of charge-density-wave oscillator
networks validate these theoretical predictions, demonstrating
practical realizability in neuromorphic hardware.

Future work includes adapting this architecture for pattern denoising
by training network weights while preserving topology and
phase-cohesiveness \cite{krotov2016DenseAssociativeMemory}, extending
the analysis to other network topologies, and experimental validation
using physical oscillator arrays and neuromorphic hardware.

\bibliographystyle{IEEEtran}
\bibliography{./bib/alias,./bib/Main,./bib/New,./bib/AO--12-30-25}

\begin{thebibliography}{10}
\providecommand{\url}[1]{#1}
\csname url@samestyle\endcsname
\providecommand{\newblock}{\relax}
\providecommand{\bibinfo}[2]{#2}
\providecommand{\BIBentrySTDinterwordspacing}{\spaceskip=0pt\relax}
\providecommand{\BIBentryALTinterwordstretchfactor}{4}
\providecommand{\BIBentryALTinterwordspacing}{\spaceskip=\fontdimen2\font plus
\BIBentryALTinterwordstretchfactor\fontdimen3\font minus
  \fontdimen4\font\relax}
\providecommand{\BIBforeignlanguage}[2]{{%
\expandafter\ifx\csname l@#1\endcsname\relax
\typeout{** WARNING: IEEEtran.bst: No hyphenation pattern has been}%
\typeout{** loaded for the language `#1'. Using the pattern for}%
\typeout{** the default language instead.}%
\else
\language=\csname l@#1\endcsname
\fi
#2}}
\providecommand{\BIBdecl}{\relax}
\BIBdecl

\bibitem{JJHopfield:82}
J.~J. Hopfield, ``Neural networks and physical systems with emergent collective
  computational abilities,'' \emph{Proceedings of the National Academy of
  Sciences}, vol.~79, no.~8, pp. 2554--2558, 1982.

\bibitem{krotov2016DenseAssociativeMemory}
D.~Krotov and J.~J. Hopfield, ``Dense {{Associative Memory}} for {{Pattern
  Recognition}},'' in \emph{Advances in {{Neural Information Processing
  Systems}}}, vol.~29.\hskip 1em plus 0.5em minus 0.4em\relax Curran
  Associates, Inc., 2016.

\bibitem{kudithipudi2025NeuromorphicComputingScale}
D.~Kudithipudi, C.~Schuman, C.~M. Vineyard, T.~Pandit, C.~Merkel, R.~Kubendran,
  J.~B. Aimone, G.~Orchard, C.~Mayr, R.~Benosman, J.~Hays, C.~Young,
  C.~Bartolozzi, A.~Majumdar, S.~G. Cardwell, M.~Payvand, S.~Buckley,
  S.~Kulkarni, H.~A. Gonzalez, G.~Cauwenberghs, C.~S. Thakur, A.~Subramoney,
  and S.~Furber, ``Neuromorphic computing at scale,'' \emph{Nature}, vol. 637,
  no. 8047, pp. 801--812, Jan. 2025.

\bibitem{LP-IG-DG:86}
L.~Personnaz, I.~Guyon, and G.~Dreyfus, ``Collective computational properties
  of neural networks: New learning mechanisms,'' \emph{Physical Review A},
  vol.~34, no.~5, p. 4217, 1986.

\bibitem{demircigil2017exponentialAssociativeMemory}
M.~Demircigil, J.~Heusel, M.~Löwe, S.~Upgang, and F.~Vermet, ``On a model of
  associative memory with huge storage capacity,'' \emph{Journal of Statistical
  Physics}, 02 2017.

\bibitem{hillar2018RobustExponentialMemory}
C.~J. Hillar and N.~M. Tran, ``Robust {{Exponential Memory}} in {{Hopfield
  Networks}},'' \emph{The Journal of Mathematical Neuroscience}, vol.~8, no.~1,
  p.~1, Dec. 2018.

\bibitem{torrejon2017NeuromorphicComputingNanoscale}
J.~Torrejon, M.~Riou, F.~A. Araujo, S.~Tsunegi, G.~Khalsa, D.~Querlioz,
  P.~Bortolotti, V.~Cros, K.~Yakushiji, A.~Fukushima, H.~Kubota, S.~Yuasa,
  M.~D. Stiles, and J.~Grollier, ``Neuromorphic computing with nanoscale
  spintronic oscillators,'' \emph{Nature}, vol. 547, no. 7664, pp. 428--431,
  Jul. 2017.

\bibitem{musa2025DenseAssociativeMemory}
K.~Musa, S.~Kumar, M.~Katidis, and Y.-P. Huang, ``Dense associative memory in a
  nonlinear optical {{Hopfield}} neural network,'' \emph{Physical Review
  Applied}, Dec. 2025.

\bibitem{choi2025HardwareImplementationRing}
W.~Choi, T.~van Bodegraven, J.~Verest, O.~Maher, D.~F. Falcone, A.~L. Porta,
  D.~Jubin, B.~J. Offrein, S.~Karg, V.~Bragaglia, and A.~{Todri-Sanial},
  ``Hardware {{Implementation}} of {{Ring Oscillator Networks Coupled}} by
  {{BEOL Integrated ReRAM}} for {{Associative Memory Tasks}},'' Mar. 2025.

\bibitem{TN-FCH-YCL:04}
T.~Nishikawa, F.~C. Hoppensteadt, and Y.~C. Lai, ``Oscillatory associative
  memory network with perfect retrieval,'' \emph{Physica D: Nonlinear
  Phenomena}, vol. 197, no. 1-2, pp. 134--148, 2004.

\bibitem{TN-YCL-FC:04}
T.~Nishikawa, Y.~C. Lai, and F.~C. Hoppensteadt, ``Capacity of oscillatory
  associative-memory networks with error-free retrieval,'' \emph{Physical
  Review Letters}, vol.~92, no.~10, p. 108101, 2004.

\bibitem{izhikevich1999weakly}
E.~M. Izhikevich, ``Weakly pulse-coupled oscillators, fm interactions,
  synchronization, and oscillatory associative memory,'' \emph{IEEE Trans.
  Neural Netw.}, vol.~10, no.~3, pp. 508--526, 1999.

\bibitem{Doerfler2014}
F.~D{\"o}rfler and F.~Bullo, ``Synchronization in complex networks of phase
  oscillators: A survey,'' \emph{Automatica}, vol.~50, no.~6, pp. 1539--1564,
  2014.

\bibitem{RWH-KK:15}
R.~W. H{\"o}lzel and K.~Katharina, ``Stability and long term behavior of a
  {Hebbian} network of {Kuramoto} oscillators,'' \emph{SIAM J. Appl. Dyn.
  Syst.}, vol.~14, no.~1, pp. 188--201, 2015.

\bibitem{XZ-ZL-XX:20}
X.~Zhao, Z.~Li, and X.~Xue, ``Stability in a {Hebbian} network of {Kuramoto}
  oscillators with second-order couplings for binary pattern retrieve,''
  \emph{SIAM J. Appl. Dyn. Syst.}, vol.~19, no.~2, pp. 1124--1159, 2020.

\bibitem{nagerl2025HigherOrderKuramotoOscillator}
J.~Nagerl and N.~G. Berloff, ``Higher-{{Order Kuramoto Oscillator Network}} for
  {{Dense Associative Memory}},'' Jul. 2025.

\bibitem{guo2025OscillatoryAssociativeMemory}
T.~Guo, A.~Ogranovich, A.~R. Venkatakrishnan, M.~R. Shapiro, F.~Bullo, and
  F.~Pasqualetti, ``Oscillatory {{Associative Memory}} with {{Exponential
  Capacity}},'' in \emph{Proceedings of the 64th {{IEEE Conference}} on
  {{Decision}} and {{Control}} ({{CDC}} 2025)}, Apr. 2025.

\bibitem{brown2004phase}
E.~Brown, J.~Moehlis, and P.~Holmes, ``On the phase reduction and response
  dynamics of neural oscillator populations,'' \emph{Neural computation},
  vol.~16, no.~4, pp. 673--715, 2004.

\bibitem{nakao2016phase}
H.~Nakao, ``Phase reduction approach to synchronisation of nonlinear
  oscillators,'' \emph{Contemporary Physics}, vol.~57, no.~2, pp. 188--214,
  2016.

\bibitem{brown2025charge}
J.~O. Brown, T.~Guo, F.~Pasqualetti, and A.~A. Balandin, ``Charge-density-wave
  quantum oscillator networks for solving combinatorial optimization
  problems,'' \emph{Physical Review Applied}, vol.~24, no.~2, p. 024040, 2025.

\bibitem{Jafarpour2019Flows}
S.~Jafarpour, E.~Y. Huang, K.~D. Smith, and F.~Bullo, ``Flow and elastic
  networks on the n-torus: Geometry, analysis, and computation,'' \emph{SIAM
  Rev.}, vol.~64, pp. 59--104, 2019.

\bibitem{PAA-RM-BA:05}
P.-A. Absil, R.~Mahony, and B.~Andrews, ``Convergence of the iterates of
  descent methods for analytic cost functions,'' \emph{SIAM Journal on Control
  and Optimization}, vol.~6, no.~2, pp. 531--547, 2005.

\end{thebibliography}

\begin{appendix}
  
  \subsection{Phase differences and stability}
  To characterize the set of stable phase-locked configurations in
  Theorem~\ref{thm:capacity}, we need to establish specific geometric
  properties of stable phase-locked configurations. In particular,
   the phase difference between neighboring oscillators must be strictly
  smaller than $\pi/2$. This property is essential to derive
  the structure of the stable configurations in
  Theorem~\ref{thm:capacity}.

  \begin{lemma}{\emph{\textbf{\emph{(Phase cohesiveness)}}}}\label{lem:phase
      cohesiveness}
    For the Kuramoto network on the honeycomb graph $\mc{G}_{m}^{n_c}$, any
    stable phase-locked configuration $\theta^*$ satisfies
    $|\theta_i^* - \theta_j^*| < \pi/2$ for all adjacent oscillators $i$ and
    $j$.  {Conversely, any phase-locked configuration with
      $|\theta_i^* - \theta_j^*| \ge \pi/2$ for some adjacent oscillators $i$
      and $j$ is unstable. \oprocend}
\end{lemma}
\begin{proof}
  The Kuramoto model with identical natural frequencies is a gradient system associated with the potential energy $U(\theta) = -\sum_{(i,j) \in \mathcal{E}} \cos(\theta_i - \theta_j)$. Consider any node $i$ in the honeycomb graph $\mc{G}_{m}^{n_c}$ and
  let $\mc{N}(i)$ denote the set of neighbors of node $i$. The Hessian matrix $H$ of the phase-locked configuration point $\theta^*$ is obtained as
  \[
      H = 
      \begin{cases} h_{ii}=
      \sum_{k \in \mathcal{N}(i)} \cos(\theta_i^* - \theta_k^*)     & \text{diagonal elements,}\\
      h_{ij}
      =-\cos(\theta_i^* - \theta_j^*) & \text{if } (i,j) \in \mathcal{E}, \\
      0 & \text{otherwise}.
      \end{cases}
  \]


  Next, we analyze the stability of $\theta^*$ by examining the change in potential energy $U(\theta)$ under small perturbations around $\theta^*$.
  Let $v = [v_1,\cdots, v_N]^\T$ be the perturbation vector around the phase-locked configuration $\theta^*$, where $N$ denotes the total number of nodes in the honeycomb graph. Using Taylor expansion the change in potential energy is given by
\begin{align*}
    \Delta U &= U(\theta^* + v) - U(\theta^*) \\
    &  = \nabla U(\theta^*)^\T v + \frac{1}{2} v^\T H v + o(\|v\|^3). 
\end{align*}
Since $\theta^*$ is a phase-locked configuration, we have $\nabla
U(\theta^*) = 0$. Therefore, the change in potential energy is
dominated by the quadratic form $v^\T H v$ for small perturbations. To
determine if $\theta^*$ is a local minimum (and thus stable), we need
to check if the quadratic form $v^\T H v$ is positive for all nonzero
perturbations $v$. We can express the quadratic form as
\begin{align*}
    &v^\T H v = \sum_{i=1}^N \sum_{j=1}^N v_i H_{ij} v_j = \underbrace{\sum_{i=1}^N v_i H_{ii} v_i}_\text{diagonal} + \underbrace{\sum_{i=1}^N \sum_{j \neq i} v_i H_{ij} v_j }_\text{off-diagonal} \\
      & = \sum_{i=1}^N  
      \sum_{j \in \mathcal{N}(i)} \!\!v_i^2 \cos(\theta_i^* - \theta_j^*)  - \sum_{i=1}^N \!\sum_{j \in \mathcal{N}(i)}\!\! v_i v_j  \cos(\theta_i^* - \theta_j^*) \\
      & = \frac{1}{2}\sum_{i=1}^N \sum_{j \in \mathcal{N}(i)}\!\!\cos(\theta_i^* - \theta_j^*) (v_i^2 - v_i v_j) 
\end{align*}
This sum counts every edge twice (once as $i \to j$ and once as
$j \to i$).  Since
$\cos(\theta_i^* - \theta_j^*) = \cos(\theta_j^* - \theta_i^*)$, we
can rewrite the sum to iterate over unique edges $(i,j)$ rather than
nodes.
To
do this, we group the $(i,j)$ term with the $(j,i)$ term:
\begin{align*}
  v^\T H v & = \sum_{(i,j) \in \mc{E}}  \Big[\cos(\theta_i^* - \theta_j^*)(v_i^2 - v_i v_j) + \\ & \;\;\;\;\;\; \cos(\theta_j^* - \theta_i^*)(v_j^2 - v_j v_i)\Big]  \\ 
  & = \sum_{(i,j) \in \mc{E}} \cos(\theta_i^* - \theta_j^*)(v_i^2 - 2 v_i v_j + v_j^2) \\
  & = \sum_{(i,j) \in \mc{E}} \cos(\theta_i^* - \theta_j^*)(v_i - v_j)^2.
\end{align*}
If $|\theta_i^* - \theta_j^*| < \pi/2$ for all $(i,j) \in \mc{E}$,
then $\cos(\theta_i^* - \theta_j^*) > 0$ everywhere. In this case,
$v^\T H v \ge 0$, vanishing only when $v_i = v_j$ for all connected
$(i,j)$ (i.e., $v \in \text{span}(\vect{1})$). Thus, $\theta^*$ is a
local minimum of the energy confirming its local exponential stability
(up to rotational symmetry).

Conversely, suppose there exists an edge $(u,w)$ on a cycle $\mc{C}_p$ where
$|\theta_u^* - \theta_w^*| > \pi/2$. We show $H$ is not positive
semi-definite. First note the equilibrium condition
$\sum_{k \in \mc{N}(i)} \sin(\theta_i^* - \theta_k^*) = 0$ at any
degree-2 node $i$ with neighbors $j,k$ implies
$\sin(\theta_j^* - \theta_i^*) = \sin(\theta_i^* - \theta_k^*)$. 
Since
every cycle in the honeycomb graph consists of nodes with degree 2
(except for nodes connecting cycles) and the shared nodes
connect cycles in a tree-like fashion, one can propagate this equality (inductively across cycles and junction nodes) to show for each cycle $\mc{C}_q$ that
$|\sin(\theta_i^* - \theta_j^*)| = c_q$ for some constant
$c_q \in [0, 1]$ and for all 
$(i,j) \in \mc{C}_q$.
     Consequently,
$\cos(\theta_i^* - \theta_j^*) = \pm \sqrt{1-c_p^2}$ for each $(i,j) \in \mc{C}_p$. Since
$|\theta_u^* - \theta_w^*| > \pi/2$, we must have
$\cos(\theta_u^* - \theta_w^*) = -{\scriptstyle \sqrt{1-c_p^2}} < 0$.

 We construct
a perturbation $v$ localized on $\mc{C}_p$ to exploit the negative curvature associated with this edge. Specifically, let $v_w = 0$ and $v_u = \delta$, where $\delta$ is an arbitrarily small scalar. For the remaining nodes in $\mc{C}_p$, we assign values that decrease linearly from $v_u$ back to $v_w$ along the path not containing the edge $(u,w)$. Since there are $n_c-1$ remaining edges, the values decrease by $\frac{\delta}{n_c - 1}$ across each edge. This
creates a ``ramp" of values around the cycle $\mc{C}_p$ that sums to
$\delta$ and has a large jump across the edge $(u,w)$. The
perturbation between node $u$ and $w$ is $\delta$, while the
perturbation between adjacent nodes on the cycle is
$\frac{\delta}{n_c-1}$. This choice of $v$ amplifies the negative
contribution from the edge $(u,w)$ while keeping contributions from
other edges small.  For any $(i,j) \notin$ $\mc{C}_p$, we keep
$v_i - v_j = 0$ such that they do not contribute to the energy
change. Thus, we can focus on the edges in $\mc{C}_p$ to evaluate
$v^\T H v$:
\vspace{-0.3em}
\begin{align}{\label{eq:negative_direction}}
\begin{aligned}
  v^\T H v &= \sum_{(i,j) \in \mc{C}_p} \cos(\theta_i^* - \theta_j^*) (v_i - v_j)^2 \\
  & = \cos(\theta_u^* - \theta_w^*) (v_u - v_w)^2 + \!\! \!\!\!\!\!\!
  \sum_{\substack{(i,j)\in \mc{C}_p\\ (i,j)\neq
      (u,w)}}\!\!\!\!\!\!\!\!
  \cos(\theta_i^* - \theta_j^*) (v_i - v_j)^2 \\
  & \le -\sqrt{1-c_p^2} \delta^2 + (n_c-1) \sqrt{1-c_p^2} \frac{\delta^2}{(n_c-1)^2} \\
  & = \delta^2\sqrt{1-c_p^2} \left( \frac{1}{n_c-1} -1\right) < 0.
   \end{aligned}
\end{align}
\vspace{0em}
Thus, $H$ has a negative direction, and $\theta^*$ is unstable. If
there are multiple edges $(i, j)$ in $\mc{C}_p$ satisfying
$|\theta_{i}^* - \theta_{j}^*| > \pi/2$, the corresponding cosine
terms $\cos(\theta_{i}^* - \theta_{j}^*)$ are all negative. Since the
squared difference terms $(v_{i} - v_{j})^2$ are non-negative, each
such edge contributes a non-positive term to the summation in
$v^\T H v$, which therefore remains negative.

Finally, consider the critical case where there exists an edge
$(u,w) \in \mc{C}_p$ such that $|\theta_u^* - \theta_w^*| =
\pi/2$. The equilibrium condition at degree-2 nodes implies
$|\sin(\theta_i^* - \theta_j^*)| = 1$ for all edges $(i,j)$ in
$\mc{C}_p$. Consequently, $\cos(\theta_i^* - \theta_j^*) = 0$ for all
$(i,j) \in \mc{C}_p$. For perturbations $v$ localized on $\mc{C}_p$,
the quadratic form $v^\T H v$ vanishes, thus we need to analyze
higher-order terms in the Taylor expansion of $\Delta U$. We have
\begin{align*}
  \Delta U^{(3)}(v) &= \frac{1}{6} \sum_{i,j,k} \frac{\partial^3 U}{\partial \theta_i \partial \theta_j \partial \theta_k} v_i v_j v_k \\
  &= -\frac{1}{6} \sum_{(i,j) \in \mc{E}} \sin(\theta_i^* - \theta_j^*) (v_i - v_j)^3 \\
  & = \frac{1}{6} \sum_{(i,j) \in \mc{E}} \sin(\theta_i^* - \theta_j^*) (v_j - v_i)^3. 
\end{align*}
 To construct a perturbation $v$ that makes
$\Delta U^{(3)}(v)$ negative, let $v_w = 0$ and $v_u = \delta > 0$. If
$\sin(\theta_u^* - \theta_w^*) = 1$, we set the differences $v_{i} - v_j$
along the path from $w$ to $u$ (excluding the edge $(u,w)$) 
equal to $\frac{\delta}{n_c-1}$. This creates a perturbation that has
a large positive jump across the edge $(u,w)$ and small negative jumps
across the other edges in the cycle. The contribution from the edge
$(u,w)$ to $\Delta U^{(3)}(v)$ is
$\frac{1}{6} \sin(\theta_u^* - \theta_w^*) (v_w - v_u)^3 =
-\frac{1}{6} (1) \delta^3 = -\frac{\delta^3}{6}$, which is
negative. The contributions from the other edges in $\mc{C}_p$ are of
order $\left(-\frac{\delta}{n_c-1}\right)^3$, which are positive but
smaller in magnitude. Summing these contributions, we have
\begin{equation}    
\begin{aligned}
  & \Delta U^{(3)}(v) = \frac{1}{6} \sum_{(i,j) \in \mc{C}_p} \sin(\theta_i^* - \theta_j^*) (v_j - v_i)^3 \\
  &= \frac{1}{6}\!\! \left[ \sin(\theta_u^* - \theta_w^*) (v_w - v_u)^3 + \!\!\!\!\!\!\!\!\!\!\sum_{\substack{(i,j)\in \mc{C}_p\\ (i,j)\neq (u,w)}}\!\!\!\!\!\!\!\! \sin(\theta_i^* - \theta_j^*) (v_j - v_i)^3 \right] \\
  &\le \frac{1}{6} \!\!\left[ (-\delta)^3 + (n_c-1)  \left(-\frac{\delta}{n_c-1}\right)^3 \right] \\ 
  &= \frac{\delta^3}{6} \left[ -1 + \frac{1}{(n_c-1)^2} \right].
\end{aligned}
\end{equation}
\end{proof}

Since $n_c \geq 5$, we have $(n_c-1)^2 > 1$, so the term in brackets
is negative, making $\Delta U^{(3)}(v) < 0$. If
$\sin(\theta_u^* - \theta_w^*) = -1$, we choose $v_u - v_w = -\delta$
instead, leading to the same result. In both cases, we can construct a
perturbation $v$ such that the potential energy decreases
locally. This implies that any configuration $\theta^*$ containing an
edge $(i, j)$ with $|\theta^*_i - \theta^*_j| = \pi/2$ cannot be a
local minimum of $U(\theta)$ and is unstable.


To conclude, we have shown that any configuration $\theta$ with
$|\theta_i - \theta_j| \geq \pi/2$ for some edge $(i,j)$ is not a
local minimum of $U$ and is therefore unstable. This completes the
proof.

\subsection{Proof of Theorem~\ref{thm:capacity}}\label{app:proof_capacity}
Let $\theta$ be the phases of the oscillators on the honeycomb graph $\mc G_m^{n_c}$ with cycles $\mc C_1,\dots,\mc C_m$. The winding number associated with the phases $\theta$ and the cycle $\mc C_p$ is
\begin{align}\label{eq: defn of winding numbers}
  q_{p}(\theta) = \frac{1}{2\pi} \sum_{\substack{(i,j) \in \mathcal{C}_p}} d_{cc}(\theta_{i},\theta_{j})
\end{align}
where
\begin{align}\label{eq: defn of cc difference}
    d_{cc}(\theta_i,\theta_j) = \theta_i - \theta_j + 2\pi q \in [-\pi,\pi[
\end{align}
is defined to be the unique real number congruent to $\theta_i - \theta_j$ modulo $2\pi$ which lies in $[-\pi,\pi[$. It is known that the winding number is always an integer \cite{Jafarpour2019Flows}. Further, $|d_{cc}(\theta_{i},\theta_{j})|$ is the distance between two angles on $\mathbb{S}^1$, so $\theta$ is phase-cohesive if and only if $|d_{cc}(\theta_{i},\theta_{j})| < \frac{\pi}{2}$ along every edge $(i,j) \in \mathcal{E}$. We first prove the following:
\begin{lemma}{\emph{\textbf{\emph{(Kirchhoff's phase-cohesive angle
          law)}}}}\label{lemma: kirchoff's phase-cohesive angle law}
  Consider the Kuramoto network on the honeycomb graph $\mathcal{G}_m^{n_c}$. If
  $|d_{cc}(\theta_{i},\theta_{j})| < \pi/2$ for all adjacent oscillators $i,j$, then
\[
|q_p(\theta)| \leq \left\lceil \frac{n_c}{4} \right\rceil - 1,
\]
for all $p \in \until{m}$.
\end{lemma}
\begin{proof}
     Because $|d_{cc}(\theta_{i},\theta_{j})| < \frac{\pi}{2}$ for all adjacent $i,j$, we have:
\begin{equation}\label{eq: bound on clockwise sum}
\begin{aligned}
  |q_p(\theta)| & = \frac{1}{2\pi} \left| \sum_{\substack{(i,j) \in \mathcal{C}_p}} d_{cc}(\theta_{i},\theta_{j})\right|  \\
  & \leq
\frac{1}{2\pi}  \sum_{\substack{(i,j) \in \mathcal{C}_p}} \left| d_{cc}(\theta_{i},\theta_{j}) \right| \\
 &< \frac{1}{2\pi} n_c \frac{\pi}{2} = \frac{n_{c}}{4}.
\end{aligned}
\end{equation}
Since $q_p(\theta)$ is an integer and from \eqref{eq: bound on clockwise sum}, we have \begin{align*}|q_p(\theta)| \leq \left\lceil \frac{n_c}{4} \right\rceil - 1. \end{align*}
\end{proof}

We are now ready to prove Theorem \ref{thm:capacity}.
We start by proving statements (i)
and (ii). Let $\theta^*$ be a phase-locked configuration satisfying
the condition (ii), that is, with constant phase differences along
each cycle, with consecutive oscillators in cycle $\mathcal{C}_p$
having phase difference $\frac{2\pi}{n_c} k_p$ for some integer
$k_p \in \{-\lceil n_c/4 \rceil + 1, \ldots, \lceil n_c/4 \rceil -
1\}$. Notice that
$\left| \frac{2\pi}{n_c} k_p \right| < \frac{\pi}{2}$ for every
$k_p \in \{-\lceil n_c/4 \rceil + 1, \ldots, \lceil n_c/4 \rceil -
1\}$:
\begin{align*}
  \max_{k_p } \left| \frac{2\pi k_p}{n_c}\right|&\leq \frac{2 \pi (\lceil \frac{n_c}{4}
                               \rceil - 1)}{n_c} < \frac{2\pi
                               \frac{n_c}{4}}{n_c} = \frac{\pi}{2}. 
\end{align*}
Then, from \cite{Doerfler2014}, due to rotational symmetry, the
Jacobian of the dynamics at $\theta^*$ has one zero eigenvalue
corresponding to uniform phase shifts. However, the phase difference
dynamics are exponentially stable: the Jacobian restricted to the
invariant subspace orthogonal to $(1,1,\ldots,1)^T$ has all
eigenvalues with negative real part, implying local exponential
stability of $\theta^*$ up to rotational equivalence. Finally, notice
the set $\mc S$ contains $N_\text{eq} = (2\lceil n_c/4 \rceil - 1)^m$,
since each $k_p$, with $p = 1,\ldots, m$, can take
$2\lceil n_c/4 \rceil - 1$ distinct values.

We now prove statement (iii).  First, Lemma \ref{lem:phase cohesiveness} implies that every stable phase-locked configuration $\theta^*$
of the honeycomb Kuramoto network satisfies
$|\theta_i^* - \theta_j^*| < \pi/2$ for all adjacent oscillators $i$
and $j$. Therefore, from Lemma \ref{lemma: kirchoff's phase-cohesive angle law}, the winding number $q_p(\theta^*)$ associated with $\theta^*$ and cycle $\mc{C}_p$ satisfies
\begin{align}{\label{eq: bound on winding numbers}}
   1-\left\lceil  \frac{n_c}{4} \right\rceil  \leq q_p(\theta^*) \leq \left\lceil \frac{n_c}{4} \right\rceil - 1.
\end{align}
Recall from \cite{Jafarpour2019Flows} that distinct phase-locked
configurations feature distinct winding vectors when the modulus of
the phase difference of adjacent oscillators is bounded by
$\pi/2$. 
Define the associated winding vector as
\begin{align*}
  q(\theta^*) = [ q_1(\theta^*),\dots, q_m(\theta^*) ].
\end{align*}
Since each $q_p(\theta^*)$, with $p = 1,\ldots, m$, can take
$2\lceil n_c/4 \rceil - 1$ distinct values, we conclude that that set
$\mc S$ contains all and only the exponentially stable phase-locked
configurations of the Kuramoto network \eqref{eq:kuramoto_identical} on
the honeycomb graph $\mathcal{G}_m^{n_c}$.


By Lemma \ref{lem:phase cohesiveness}, any stable phase-locked
configuration on the honeycomb graph must be phase-cohesive, meaning
$|\theta^*_i - \theta^*_j| < \pi/2$ for all adjacent nodes $i,
j$. Furthermore, as established in \cite{Jafarpour2019Flows}, distinct
phase-cohesive configurations are uniquely characterized by their
winding vectors. Consequently, the number of stable configurations is
bounded by the number of admissible winding vectors.

Following \eqref{eq: bound on winding numbers}, there are at most $2 \lceil n_{c}/4 \rceil - 1$ possible integer values for $q_{p}(\theta^*)$ corresponding to a phase-cohesive configuration. Since the honeycomb graph is formed by $m$ independent cycles, the total number of distinct phase-cohesive phase-locked configurations is at most
  \[
  \prod_{i=1}^m \left(2\left\lceil \frac{n_c}{4} \right\rceil - 1\right) = \left(2\left\lceil \frac{n_c}{4} \right\rceil - 1\right)^m = N_{\text{eq}}.
  \]
  Because we have already identified exactly $N_{\text{eq}}$ stable
  configurations in part (i),  no other stable
  configuration exists.

  \subsection{Proof of Theorem~\ref{thm:basin}}\label{app:proof_basin}

{
To prove Theorem~\ref{thm:basin}, we first establish the following general result on Kuramoto-model basins of attraction using Lyapunov-like arguments.

\begin{proposition}
{\emph{\textbf{\emph{(Infinity-Norm Weak Lyapunov Function)}}}}
\label{prop:infinity norm nonexpansion}
Let $\theta^*$ denote a phase-locked configuration of the Kuramoto oscillator network \eqref{eq:kuramoto_identical} on a symmetric, connected graph $\mathcal G=(\mathcal V,\mathcal E)$ with positive weights. Suppose for some $\gamma > 0$ that $\max_{(i,j)\in \mathcal E}|\theta_i^*-\theta_j^*| \le \frac{\pi}{2}-\gamma$ and further suppose $\theta(0)$ satisfies $\max_{i\in \mathcal{V}}|\theta_i(0)-\theta_i^*| < \gamma$. The following hold:
\begin{enumerate}
    \item $\max_{i\in \mathcal{V}}|\theta_i(t)-\theta_i^*|$ is nonincreasing in $t$.
    \item $\theta(0)$ is in a phase-locked configuration if and only if $\theta(0)$ is rotationally equivalent to $\theta^*$.
    \item As $t\rightarrow \infty$, $\theta(t)$ converges to a phase-locked configuration rotationally equivalent to $\theta^*$.
\end{enumerate}
\end{proposition}

\begin{proof}
We first prove (i). Let $V(\theta(t)) = \max_{k \in \mathcal{V}}|\theta_k(t) - \theta_k^*|$.
If $\theta(t) = \theta^*$, then $V(\theta(t)) = 0$, which is the global minimum and trivially nonincreasing.
Therefore, we assume $\theta(t) \neq \theta^*$, so $V(\theta(t)) > 0$.

Assume at time $t$ that $V(\theta(t)) < \gamma$. Let $i \in \mathcal{V}$ be such that $|\theta_i(t) - \theta_i^*| = V(\theta(t))$ (note that $i$ may not be unique). We evaluate the upper right Dini derivative $D^+V(\theta(t))$ \tguomargin{We need to introduce this}.
The analysis is broken into two cases: $\theta_i(t) > \theta_i^*$ and $\theta_i(t) < \theta_i^*$.

Let $\theta_i(t) < \theta_i^*$ (the case $\theta_i(t) > \theta_i^*$ is handled similarly using a symmetric argument).
Because $i$ holds the maximum negative deviation, $\theta_i(t) - \theta_i^* = -V(\theta(t))$. The neighboring oscillators $j$ satisfy $\theta_j^* - \theta_j(t) \le |\theta_j^* - \theta_j(t)| \le V(\theta(t)) = \theta_i^* - \theta_i(t)$ and, consequently,
\begin{equation}\label{eq: lower bound}
    \theta_j(t) - \theta_i(t) \ge \theta_j^* - \theta_i^*.
\end{equation}
We next proceed to establish an upper bound on $\theta_j(t) - \theta_i(t)$. By splitting up the difference $\theta_j(t) - \theta_i(t)$, we obtain
\begin{equation}
    \theta_j(t) - \theta_i(t) = (\theta_j(t) - \theta_j^*) + (\theta_j^* - \theta_i^*) + (\theta_i^* - \theta_i(t)).
\end{equation}
Since $\theta_j(t) - \theta_j^* \le |\theta_j(t) - \theta_j^*| \le V(\theta(t))$ and $\theta_i^* - \theta_i(t) = V(\theta(t))$, we can bound this sum by
\begin{equation}
    \theta_j(t) - \theta_i(t) \le 2V(\theta(t)) + (\theta_j^* - \theta_i^*).
\end{equation}
Using the assumption $V(\theta(t)) < \gamma$, we obtain:
\begin{equation}
    \theta_j(t) - \theta_i(t) < 2\gamma + (\theta_j^* - \theta_i^*).
\end{equation}
From the proposition hypotheses, we know $\theta_j^* - \theta_i^* \le \frac{\pi}{2} - \gamma$, which rearranges to $2\gamma \le \pi - 2(\theta_j^* - \theta_i^*)$. Substituting this into the inequality yields
\begin{equation}\label{eq: upper bound}
    \theta_j(t) - \theta_i(t) < \pi - 2(\theta_j^* - \theta_i^*) + (\theta_j^* - \theta_i^*) = \pi - (\theta_j^* - \theta_i^*).
\end{equation}
Combining the lower \eqref{eq: lower bound} and upper \eqref{eq: upper bound} bounds we conclude
\begin{equation}
    \theta_j^* - \theta_i^* \le \theta_j(t) - \theta_i(t) < \pi - (\theta_j^* - \theta_i^*).
\end{equation}
Since $|\theta_j^* - \theta_i^*| \le \frac{\pi}{2} - \gamma < \frac{\pi}{2}$ by hypothesis, and $\sin(x) \ge \sin(a)$ whenever $a \le x < \pi - a$ for $|a| < \frac{\pi}{2}$ (see Figure \ref{fig: visualization of sin monotonicity property} for a visualization), we conclude that:
\begin{equation}
    \sin(\theta_j(t) - \theta_i(t)) \ge \sin(\theta_j^* - \theta_i^*)
\end{equation}
with equality holding only when $\theta_j(t) - \theta_i(t) = \theta_j^* - \theta_i^*$.
Consequently, the dynamics of node $i$ satisfy:
\begin{equation}\label{eq: full velocity inequality for the maximally perturbed}
    \dot{\theta}_i(t) = \sum_{j \in N_i} a_{ij} \sin(\theta_j(t) - \theta_i(t)) \ge \sum_{j \in N_i} a_{ij} \sin(\theta_j^* - \theta_i^*) = 0
\end{equation}
where the last equality follows from the fact that $\theta^*$ is an equilibrium, and equality holds if and only if $\theta_j(t) - \theta_i(t) = \theta_j^* - \theta_i^*$ for every neighbor $j$ of $i$. Since $\dot \theta_i(t) \ge 0$ and since $\theta_i(t) < \theta_i^*$ we determine $D^+ |\theta_i(t)-\theta_i^*| \leq 0$.

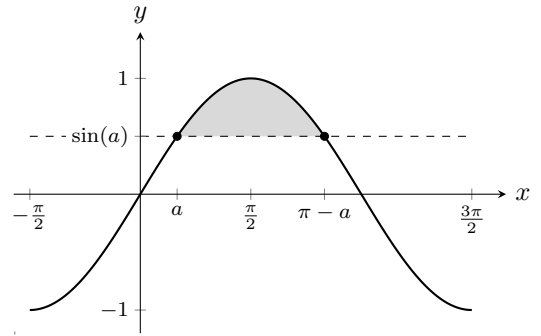
\begin{figure}[htbp]
    \centering
    \begin{tikzpicture}
        \begin{axis}[
            width=3.2in,
            height=2.2in,
            axis lines=middle,
            axis on top=true, 
            xmin=-1.8, xmax=5.2,
            ymin=-1.2, ymax=1.4,
            xtick={-1.5708, 0, 0.5236, 1.5708, 2.6180, 4.7124},
            xticklabels={$-\frac{\pi}{2}$, $0$, $a$, $\frac{\pi}{2}$, $\pi-a$, $\frac{3\pi}{2}$},
            ytick={-1, 0.5, 1},
            yticklabels={$-1$, $\sin(a)$, $1$},
            yticklabel style={fill=white, inner sep=2pt},
            xticklabel style={fill=white, inner sep=2pt}, 
            xlabel={$x$},
            ylabel={$y$},
            every axis x label/.style={at={(ticklabel* cs:1.0)}, anchor=west},
            every axis y label/.style={at={(ticklabel* cs:1.0)}, anchor=south},
            ticklabel style={font=\footnotesize},
            enlargelimits=false,
            clip=false,
        ]
            \addplot[name path=sine, domain=-1.5708:4.7124, samples=150, thick] {sin(deg(x))};
            
            \addplot[name path=line, domain=-1.5708:4.7124, dashed] {0.5};
            
            \addplot [fill=gray!30] fill between[of=sine and line, soft clip={domain=0.5236:2.6180}];
            
            \addplot[mark=*, mark size=1.5pt, only marks] coordinates {(0.5236, 0.5) (2.6180, 0.5)};
            
            \draw[dashed, gray] (1.5708, -1.2) -- (1.5708, 1.4);
            
        \end{axis}
    \end{tikzpicture}
    \caption{Visualization of the inequality $\sin(x) \geq \sin(a)$ over the interval $x \in [a, \pi-a]$.}
    \label{fig: visualization of sin monotonicity property}
\end{figure}

Then, since $D^+ |\theta_i(t)-\theta_i^*| \leq 0$ for any $i$ satisfying $|\theta_i(t)-\theta_i^*| = \max_{k \in \mathcal{V}}|\theta_k(t) - \theta_k^*|$, we conclude that $D^+V(\theta(t)) \le 0,$ meaning $\max_{k \in \mathcal{V}}|\theta_k(t) - \theta_k^*|$ is nonincreasing in $t$. This proves (i).

We now prove (ii). \emph{If direction ($\Leftarrow$):} This direction holds by the fact that any configuration rotationally equivalent to the phase-locked configuration $\theta^*$ must also be phase-locked.

\emph{Only if direction ($\Rightarrow$):} We again focus on the case where there exists a node $i$ satisfying $\theta^*_i - \theta_i(0) = \max_{k \in \mathcal{V}} |\theta_k(0)-\theta^*_k|$. We will show inductively that every node $j$ also satisfies $\theta_j(0)-\theta^*_j = \max_{k \in \mathcal{V}} |\theta_k(0)-\theta^*_k|$, which implies that $\theta(0)-\theta^*=(\max_{k \in \mathcal{V}} |\theta_k(0)-\theta^*_k|)(1,...,1)^{\top}$ and hence that $\theta(0)$ is rotationally equivalent to $\theta^*$.

As a base case for induction we have that $\theta^*_i - \theta_i(0) = \max_{k \in \mathcal{V}} |\theta_k(0)-\theta^*_k|$. For the inductive step, suppose a node $j$ satisfies $\theta_j(0)-\theta^*_j = \max_{k \in \mathcal{V}} |\theta_k(0)-\theta^*_k|$. Note that any phase-locked configuration of \eqref{eq:kuramoto_identical} is at equilibrium, so $\dot \theta_j = 0$. As discussed in the proof of part (i), we know equality holds in \eqref{eq: full velocity inequality for the maximally perturbed} if and only if every neighbor $\ell$ of $j$ satisfies $\theta_{j}(0)-\theta_{\ell}(0)=\theta^*_j-\theta^*_{\ell}$, or equivalently that $\theta_j(0)-\theta^*_j=\theta_{\ell}(0)-\theta^*_{\ell}$; further, the equality case must hold because $\dot\theta_j(0)=0$ by hypothesis. This implies that each neighbor $\ell$ of $j$ must satisfy $\theta_{\ell}(t)-\theta^*_{\ell} = \max_{k \in \mathcal{V}} |\theta_k(t)-\theta^*_k|$. Thus, any $j$ satisfying $\theta_j(0)-\theta^*_j = \max_{k \in \mathcal{V}} |\theta_k(0)-\theta^*_k|$ only has neighbors $\ell$ which satisfy the same property $\theta_{\ell}(t)-\theta^*_{\ell} = \max_{k \in \mathcal{V}} |\theta_k(t)-\theta^*_k|$; since the underlying graph is connected, we conclude by induction that every node $j$ satisfies $\theta_j(0)-\theta^*_j = \max_{k \in \mathcal{V}} |\theta_k(0)-\theta^*_k|$. As discussed in the first paragraph, $\theta(0)$ is rotationally equivalent to $\theta^*$ as desired.

We now prove (iii). We know that the Kuramoto oscillator network with symmetric weights is a gradient flow on $\R^{|\mathcal{V}|}$ with respect to an analytic energy function. Note that the open ball $B = \{\theta: \max_k |\theta_k-\theta_k^*| < \gamma\}$ is forward-invariant by part (i), since the infinity-norm distance to $\theta^*$ cannot increase (in particular, increase above $\gamma$) under the dynamics when starting in $B$. Then Theorem 2.2 of \cite{PAA-RM-BA:05} (regarding bounded analytic gradient flows) implies that any trajectory $\theta(t)$ starting in $B$ must converge to a unique equilibrium point. Further, since $\theta(t) \in B$ for all $t \geq 0$ the limit either lies in $B$ or on its boundary. To see that $\lim_{t \rightarrow \infty} \theta(t)$ lies in $B$ as opposed to on its boundary, note that $\| \theta(t) -\theta^*\|_{\infty}$ is nonincreasing with time; since $\|\theta(0)-\theta^*\| < \gamma$ we know that $\lim_{t \rightarrow \infty} \| \theta(t) -\theta^*\|_{\infty} < \gamma$ and thus the limit stays in $B$. Finally, we know from part (ii) that all equilibria (which are phase-locked configurations)
inside $B$ must be rotationally equivalent to $\theta^*$ and hence $\lim_{t \rightarrow \infty} \theta(t)$ is a phase-locked configuration rotationally equivalent to $\theta^*$.
\end{proof}

We are now ready to prove Theorem \ref{thm:basin} by combining the stability guarantees of Proposition \ref{prop:infinity norm nonexpansion} with our characterization of the honeycomb's stable equilibria from Theorem \ref{thm:capacity}

\begin{proof}[Proof of Theorem \ref{thm:basin}]
From Theorem III.1, adjacent oscillators in a stable phase-locked configuration $\theta^*$ have a maximum phase difference bounded by $\frac{2\pi}{n_c}(\lceil \frac{n_c}{4} \rceil - 1)$. Thus, the inequality in \eqref{eq:basin_bound} directly follows by applying result (iii) in Proposition \ref{prop:infinity norm nonexpansion}.

We now show the last equality of \eqref{thm:basin}. Notice that
  \begin{align*}
    &\frac{\pi}{2} - \left(\lceil \frac{n_c}{4} \rceil - 1\right)\frac{2\pi}{n_c} 
    = \frac{\pi n_c}{2n_c} - \frac{2\pi\lceil n_c/4 \rceil}{n_c} + \frac{2\pi}{n_c}\\
    &= \frac{\pi n_c + 4\pi - 4\pi\lceil n_c/4 \rceil}{2n_c} 
    = \frac{\pi(n_c + 4 - 4\lceil n_c/4 \rceil)}{2n_c}.
  \end{align*}
  Thus, we need to show that $n_c + 4 - 4\lceil n_c/4 \rceil = r$. We
  proceed by cases based on $n_c \bmod 4$:

  \noindent
  \text{Case 1:} $n_c = 4q$ for some integer $q \geq 1$. Then:
  \begin{itemize}
  \item $(n_c - 1) \bmod 4 = 3$, so $r = 3 + 1 = 4$,
  \item $\lceil n_c/4 \rceil = q$,
  \item $n_c + 4 - 4\lceil n_c/4 \rceil = 4q + 4 - 4q = 4 = r$.
  \end{itemize}

  \noindent
  \text{Case 2:} $n_c = 4q + 1$ for some integer $q \geq 0$. Then:
  \begin{itemize}
  \item $(n_c - 1) \bmod 4 = 0$, so $r = 0 + 1 = 1$,
  \item $\lceil n_c/4 \rceil = q + 1$,
  \item $n_c + 4 - 4\lceil n_c/4 \rceil = 4q + 1 + 4 - 4(q+1) = 1 = r$.
  \end{itemize}

  \noindent
  \text{Case 3:} $n_c = 4q + 2$ for some integer $q \geq 0$. Then:
  \begin{itemize}
  \item $(n_c - 1) \bmod 4 = 1$, so $r = 1 + 1 = 2$,
  \item $\lceil n_c/4 \rceil = q + 1$,
  \item $n_c + 4 - 4\lceil n_c/4 \rceil = 4q + 2 + 4 - 4(q+1) = 2 = r$.
  \end{itemize}

  \noindent
  \text{Case 4:} $n_c = 4q + 3$ for some integer $q \geq 0$. Then:
  \begin{itemize}
  \item $(n_c - 1) \bmod 4 = 2$, so $r = 2 + 1 = 3$,
  \item $\lceil n_c/4 \rceil = q + 1$,
  \item $n_c + 4 - 4\lceil n_c/4 \rceil = 4q + 3 + 4 - 4(q+1) = 3 = r$.
  \end{itemize}
  In all cases, $n_c + 4 - 4\lceil n_c/4 \rceil = r$, completing the proof.
\end{proof}
}

\end{appendix}

\end{document}